\documentclass[10pt,twocolumn,letterpaper]{article}

\usepackage{microtype}
\usepackage{graphicx}
\usepackage{subfigure}
\usepackage{natbib}
\usepackage{hyperref}

\usepackage{amsmath}
\usepackage{amssymb}
\usepackage{mathtools}
\usepackage{amsthm}

\usepackage[capitalize,noabbrev]{cleveref}

\theoremstyle{plain}
\newtheorem{theorem}{Theorem}[section]
\newtheorem{proposition}[theorem]{Proposition}

\theoremstyle{definition}

\theoremstyle{remark}

\usepackage[margin=20mm,columnsep=10mm]{geometry}

\usepackage[utf8]{inputenc} 
\usepackage[T1]{fontenc}    
\usepackage{hyperref}       
\usepackage{url}            
\usepackage{booktabs}       
\usepackage{amsfonts}       
\usepackage{nicefrac}       
\usepackage{microtype}      
\usepackage{xcolor}         

\usepackage{times}
\usepackage{epsfig}
\usepackage{graphicx}
\usepackage{amsmath}
\usepackage{amssymb}
\usepackage{url}            
\usepackage{booktabs}       
\usepackage{amsfonts}       
\usepackage{nicefrac}       
\usepackage{microtype}      
\usepackage{commath}
\usepackage{mathtools,amsbsy}
\usepackage[inline]{enumitem} 
\usepackage[all]{foreign}
\usepackage{amsthm}
\usepackage{wrapfig}

\DeclareMathOperator{\SE}{\textup{SE}}
\DeclareMathOperator{\SO}{\textup{SO}}

\DeclareMathOperator{\xm}{\textup{\textbf{x}}}
\DeclareMathOperator{\cm}{\textup{\textbf{c}}}
\DeclareMathOperator{\ym}{\textup{\textbf{y}}}
\DeclareMathOperator{\Rb}{\textbf{\textit{R}}}

\DefineNamedColor{named}{myblue}{RGB}{68,114,196}
\DefineNamedColor{named}{mygreen}{RGB}{112,173,71}
\DefineNamedColor{named}{mygold}{RGB}{255,192,0}
\DefineNamedColor{named}{myorange}{RGB}{237,125,49}

\begin{document}

\title{\textbf{Steerable 3D Spherical Neurons}}

\author{%
	Pavlo Melnyk,\quad Michael Felsberg,\quad Mårten Wadenbäck\\
	{\small Computer Vision Laboratory, Department of Electrical Engineering, Linköping University}\\
	{\small \texttt{\{pavlo.melnyk, michael.felsberg, marten.wadenback\}@liu.se}}
}

\date{}

\maketitle










\begin{abstract}
Emerging from low-level vision theory, steerable filters found their counterpart in prior work on steerable convolutional neural networks equivariant to rigid transformations. In our work, we propose a steerable feed-forward learning-based approach that consists of neurons with spherical decision surfaces and operates on point clouds.
Such spherical neurons are obtained by conformal embedding of Euclidean space and have recently been revisited in the context of learning representations of point sets. 
Focusing on 3D geometry, we exploit the isometry property of spherical neurons and derive a 3D steerability constraint. 
After training spherical neurons to classify point clouds in a canonical orientation, we use a tetrahedron basis to quadruplicate the neurons and construct rotation-equivariant spherical filter banks.
We then apply the derived constraint to interpolate the filter bank outputs and, thus, obtain a rotation-invariant network. 
Finally, we use a synthetic point set and real-world 3D skeleton data to verify our theoretical findings.
The code is available at \url{https://github.com/pavlo-melnyk/steerable-3d-neurons}.
\end{abstract}

\section{Introduction}
\label{introduction}
We present a feed-forward model consisting of steerable 3D neurons for point cloud classification, an important and challenging problem with many applications such as autonomous driving, human-robot interaction, and mixed-reality installations. 
Constructing rotation-equivariant (steerable) models allows us to use the feature of a given point cloud to synthesize features of the same point cloud with different orientations.
Further, rotation-equivariant networks enable producing rotation-invariant predictions and, therefore, lowered data augmentation requirements for learning. 

In this paper, we achieve the steerability using a conformal embedding to obtain higher-order decision surfaces. Following the motivation in the recent work of \citet{melnyk2020embed}, we focus on 3D geometry and spherical decision surfaces, arguing for their natural suitability for problems in Euclidean space. We show how a \textit{spherical neuron}, \ie, the hypersphere neuron \citep{banarer2003hypersphere} or its generalization for 3D input point sets --- the geometric neuron \citep{melnyk2020embed} --- can be turned into a steerable neuron. 
Making spherical neurons \textit{steerable} adds to the practical value of the prior work by \citet{melnyk2020embed}.
\begin{figure}
	\centering
	\includegraphics[width=\linewidth]{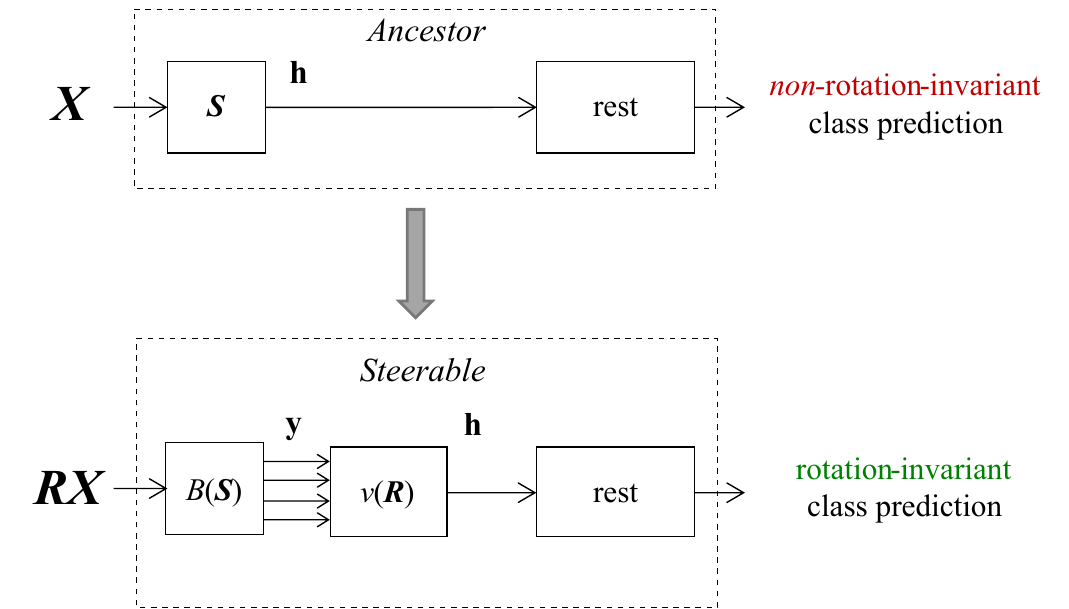}
	\caption{Our approach overview:  First, we train a classifier (\textit{Ancestor}) consisting of the spherical neurons (necessarily in the first layer) to classify 3D point clouds in a canonical orientation. We then fix all the learned parameters and construct rotation-equivariant spherical filter banks $B(\textbf{\textit{S}})$ \eqref{eq:sphere_filter_bank} from the first layer parameters. Next, we compute the interpolation coefficients $v(\textbf{\textit{R}})$ \eqref{eq:vs_for_one_point} to fulfill the steerability constraint \eqref{eq:our_steerability_constraint}. The result is a \textit{Steerable} classifier, where \textbf{y} is equivariant and \textbf{h} is invariant to rotations of the model input.}
	\label{fig:model_overview}
	\vspace{-9pt}
\end{figure}

We prove that the aforementioned spherical neurons in any dimension require only up to first-degree spherical harmonics to accommodate the effect of rotation.
This allows us to derive a 3D steerability constraint for such neurons and to describe a recipe for creating a steerable model from a pretrained classifier (see Figure~\ref{fig:model_overview}). 
Using the synthetic Tetris dataset \citep{thomas2018tensor} and the skeleton data from the UTKinect-Action3D dataset \citep{xia2012view}, we verify the derived constraint and check its stability with respect to perturbations in the input. 

Importantly, we focus on theoretical aspects of our steerable method to produce rotation-invariant predictions, and show that it is attainable by using both synthetic and real 3D data (Section~\ref{known_r_experiment}). What remains to do is to devise a practical way to achieve it. 

The core of our contributions are the novel theoretical results on steerability and equivariance and can be summarized as follows:
 
(\textbf{a}) We prove that the activation of spherical neurons \citep{banarer2003hypersphere, melnyk2020embed} on rotated input only varies by up to first-degree spherical harmonics in the rotation angle (Section~\ref{steerable_basis}).

(\textbf{b}) Based on a minimal set of four spherical neurons that are rotated to the corresponding vertices of a regular tetrahedron, we construct a rotation-equivariant spherical filter bank (Section~\ref{filter_bank}) and derive the main result of our paper (Section~\ref{our_steerability_constraint}) --- a 3D steerability constraint for spherical neurons \eqref{eq:our_steerability_constraint} and \eqref{eq:gn_steerability_constraint}.

\section{Related work}
\subsection{Steerability and equivariance}
\label{related_work_1}
Steerability is a powerful concept from early vision and image processing \citep{freeman1991design, knutsson1992aframework, simoncelli1992shiftable, perona1995deformable, simoncelli1995pyramid, teo1998lie} that resonates in the era of deep learning.
The utility of steerable filters and the main theorems for their construction are presented in the seminal work of \citet{freeman1991design}. 

Geometric equivariance in the context of computer vision has been an object of extensive research over the last decades \citep{VANGOOL1995259}. 
Equivariance is a necessary property for steerability since steerability requires changing the function output depending on the actions of a certain group.
Thus, the output needs to contain a representation of the group that acts on the input.
Therefore, the operator must commute with the group (note that the representation might change from input to output).
For this reason, Lie theory can be used to study steerable filters for equivariance and invariance \citep{reisert2008group}.

Nowadays, equivariance-related research extends into deep learning, \eg, the $\SE(3)$-equivariant models of \citet{fuchs2020se3}, \citet{thomas2018tensor}, and \citet{zhao2020quaternion}, and the $\SO(3)$-equivariant network of \citet{anderson2019cormorant}, as well as \citet{Marcos_2017_ICCV} and \citet{NEURIPS2018_a3fc981a}. Also steerable filter concepts are increasingly used in current works. For example, the work of 
\citet{cohen2016steerable} considered image data and proposed a CNN architecture to produce equivariant representations with steerable features, which involves fewer parameters than traditional CNNs. They considered the dihedral group (\textit{discrete} rotation, translation, and reflection), and the steerable representations in their work are proposed as formation of elementary feature types. One limitation of their approach is that rotations are restricted to four orientations, \ie, by multiples $\pi/2$.
More recently, \citet{weiler2018learning} utilized group convolutions and introduced steerable filter convolutional neural networks (SFCNNs) operating on images to jointly attain equivariance under translations and discrete rotations. In their work, the filter banks are learned rather than fixed. 
Further, the work of \citet{weiler2019general} proposed a unified framework for E(2)-equivariant steerable CNNs and presented their general theory.

The steerable CNNs for 3D data proposed by \citet{weiler20183d} are closely related to our work. The authors employed a combination of scalar, vector and tensor fields as features transformed by $\SO(3)$ representations and presented a model that is equivariant to $\SE(3)$ transformations. They also considered different types of nonlinearities suitable for nonscalar components of the feature space.
The novel SE(3)-equivariant approach by \citet{fuchs2020se3} introduced a self-attention mechanism that is invariant to global rotations and translations of its input and solves the limitation of angularly constrained filters in other equivariant methods, \eg \citet{thomas2018tensor}.
Noteworthy, the work of \citet{jing2020learning} proposed the geometric vector perceptron (GVP) consisting of two linear transformations for the input scalar and vector features, followed by nonlinearities. The GVP scalar and vector outputs are invariant and equivariant, respectively, with respect to an arbitrary composition of rotations and reflections in the 3D Euclidean space.
\vspace{-2pt}

\subsection{Conformal modeling and the hypersphere neuron}
The utility of conformal embedding for Euclidean geometry and the close connection to Minkowski spaces  are thoroughly discussed by \citet{li2001generalized}. An important result is that one can construct hyperspherical decision surfaces using representations in the conformal space \citep{li2001universal}, as done in the work of \citet{perwass2003spherical}. 
The hypersphere neuron proposed by \citet{banarer2003hypersphere} is such a spherical classifier.
Remarkably, since a hypersphere can be seen as a generalization of a hyperplane, the standard neuron can be considered as a special case of the hypersphere neuron.

Stacking multiple hypersphere neurons in a feed-forward network results in a multilayer hypersphere perceptron (MLHP), which was shown by \citet{banarer2003design} to outperform the standard MLP for some classification tasks. However, its application to point sets was not discussed.
This motivated the work of \citet{melnyk2020embed} on the geometric neuron, where the learned parameters were shown to represent a combination of spherical decision surfaces. Moreover, the geometric (and hypersphere) neuron activations were proved to be isometric in the 3D Euclidean space. In our work, we use the latter observation as the necessary condition for deriving the steerability constraint.
\vspace{-2pt}

\subsection{Comparison to other methods}
Steerability requires filters that either are (combinations of) spherical harmonics (see, \eg, \citet{fuchs2020se3}) or are constructed based on \textit{learned} neurons and that behave as such (our work).

The key point distinguishing our approach from other equivariant networks is that we do not constrain the space of learnable parameters (as opposed to, \eg, \citet{thomas2018tensor}, \citet{weiler20183d}, and \citet{fuchs2020se3}), but construct our steerable model from a freely trained base network, as we discuss in detail in Section~\ref{methodology}. That is, the related work uses spherical harmonics as atoms, \ie, a hand-designed basis, and learns \textit{only} the linear coefficients under constraints. In contrast, the only thing we inherit from the hand design is the constraint of first-degree harmonics (see Theorem~\ref{th:the_theorem}) -- all other degrees of freedom are learnable.

In addition, even though we only consider scalar fields in the present work, the theoretical results of Section~\ref{methodology} can be applied to broader classes of feature fields.

\section{Background}
\label{background}

\subsection{Steerability}
\label{steerability}
As per \citet{freeman1991design}, a 2D function $f(x, y)$ is said to steer if it can be written as a linear combination of rotated versions of itself, i.e., when it satisfies the constraint
\begin{equation}
	\label{eq:2d_steerability_constraint}
	f^\theta(x, y) = \sum_{j=1}^{M}v_j(\theta) f^{\theta_{j}}(x, y)~,
\end{equation}
where $v_j(\theta)$ are the interpolation functions,  $\theta_{j}$ are the basis function orientations, and $M$ is the number of basis function terms required to steer the function.
An alternative formulation can be found in the work of \citet{knutsson1992aframework}.
In 3D, the steering equation becomes
\begin{equation}
	\label{eq:3d_steerability_constraint}
	f^{\textbf{\textit{R}}}(x, y, z) = \sum_{j=1}^{M}v_j(\textbf{\textit{R}}) f^{\textbf{\textit{R}}_{j}}(x, y, z)~,
\end{equation}
where $f^{\textbf{\textit{R}}}(x, y, z)$ is $f(x, y, z)$ rotated by $\textbf{\textit{R}} \in \SO(3)$, and each $\textbf{\textit{R}}_j \in \SO(3)$ orients the corresponding $j$th basis function.

Theorems 1, 2 and 4 in \citet{freeman1991design} describe the conditions under which the steerability constraints ({\ref{eq:2d_steerability_constraint}}) and ({\ref{eq:3d_steerability_constraint}}) hold, and how to determine the minimum number of basis functions for the 2D and 3D case, respectively.

\subsection{Conformal embedding}
\label{conformal_embedding}
We refer the reader to Section 3 in the work of \citet{melnyk2020embed} for more details, and only briefly introduce important notation in this section.
 The conformal space for the Euclidean $\mathbb{R}^{n}$ counterpart can be formed as $\mathbb{ME}^{n} \equiv \mathbb{R}^{n+1,1} = \mathbb{R}^{n} \oplus \mathbb{R}^{1,1}$, where $\mathbb{R}^{1, 1}$ is the Minkowski plane \citep{li2001generalized} with orthonormal basis defined as $\{e_{+},~e_{-}\}$ and null basis   $\{e_{0},~e_{\infty}\}$ representing the origin $e_{0} = \frac{1}{2} (e_{-} - e_{+})$ and point at infinity $e_{\infty} = e_{-} + e_{+}$. Thus, a Euclidean vector $\textbf{x}\in\mathbb{R}^n$ can be embedded in the conformal space $\mathbb{ME}^n$  as
\begin{equation}
	\label{eq:conformal_embedding}
	X = \mathcal{C}(\textbf{x}) = \textbf{x} + \frac{1}{2}\lVert\textbf{x}\rVert^2~e_{\infty}+e_{0}~,
\end{equation}
where $X \in \mathbb{ME}^n$ is called \textit{normalized}. 
The conformal embedding \eqref{eq:conformal_embedding} represents the stereographic projection of $\textbf{x}$ onto a projection sphere in $\mathbb{ME}^n$ and is \textit{homogeneous}, \ie, all embedding vectors in the equivalence class $[X] = \big\{Z\in\mathbb{R}^{n+1,1} : Z=\gamma X,\;\gamma\in\mathbb{R} \setminus\{0\}\big\}$ are taken to represent the same vector $\textbf{x}$.
Importantly, given the conformal embedding $X$ and some $Y=\textbf{y} + \frac{1}{2}\lVert\textbf{y}\rVert^2~e_{\infty}+e_{0}$, their scalar product in the conformal space corresponds to their Euclidean distance, $X\cdot Y = -\frac{1}{2}\lVert\textbf{x}-\textbf{y}\rVert^2$. This interpretation of the scalar product in the conformal space is the main motivation in constructing spherical neurons.

\subsection{Spherical neurons}
\label{spherical_classifiers}
\begin{figure}
	\centering
	\includegraphics[width=0.9\linewidth]{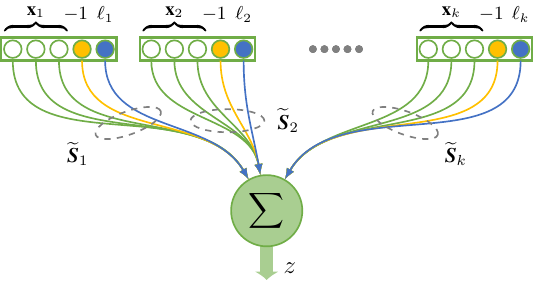}
	\caption{The geometric neuron ($\ell_k = -\frac{1}{2}\lVert\textbf{x}_k\rVert^2$).}
	\label{fig:geometric_neuron}
	\vspace{-9pt}
\end{figure}
By \textit{spherical neurons}, we collectively refer to hypersphere \citep{banarer2003design} and geometric \citep{melnyk2020embed} neurons, which have spherical decision surfaces.

As discussed by \citet{banarer2003design}, by embedding both a data vector $\textbf{x} \in\mathbb{R}^n$ and a hypersphere $S\in\mathbb{ME}^n$ in $\mathbb{R}^{n+2}$ as
\begin{equation}
	\label{hypersphere_in_r}
	\begin{aligned}
		\textbf{\textit{X}} &= \big(x_1, \dots, x_n, -1, -\frac{1}{2}\lVert\textbf{x}\rVert^2\big)\in\mathbb{R}^{n+2},\\
		\textbf{\textit{S}} &= \big(c_1, \dots, c_n, \frac{1}{2}(\lVert\textbf{c}\rVert^2 - r^2), 1\big)\in\mathbb{R}^{n+2},
	\end{aligned}
\end{equation}
where  $\textbf{c} =(c_1, \dots, c_n)\in\mathbb{R}^n$ is the hypersphere center and $r\in\mathbb{R}$ is the radius, their scalar product $X \cdot S$ in the conformal space $\mathbb{ME}^n$ can be computed equivalently in $\mathbb{R}^{n+2}$ as $\textbf{\textit{X}}^\top \textbf{\textit{S}}$:
\begin{equation}
    \label{scalar_product_isomorphism}
    X \cdot S = \textbf{\textit{X}}^\top \textbf{\textit{S}} = -\frac{1}{2}\norm{\textbf{x}-\textbf{c}}^2 + \frac{1}{2}r^2~.
\end{equation}
  This result enables the implementation of a hypersphere neuron in $\mathbb{ME}^n$ using the standard dot product in $\mathbb{R}^{n+2}$. The hypersphere vector components are treated as independent learnable parameters during training. Thus, a spherical classifier effectively learns \textit{non-normalized} hyperspheres of the form $\widetilde{\textbf{\textit{S}}} = ({s_1}, \dots, {s_{n+2}}) \in \mathbb{R}^{n+2}$.
 Due to the homogeneity of the representation, both normalized and non-normalized hyperspheres represent the same decision surface. More details can be found in Section 3.2 in the work of \citet{melnyk2020embed}. 

The geometric neuron is a generalization of the hypersphere neuron for point sets as input, see Figure~\ref{fig:geometric_neuron}. 
A single geometric neuron output is thus the sum of the signed distances of $K$ input points to $K$ learned hyperspheres
\begin{equation}
	\label{eq:geometric_neuron}
	z = \sum_{k=1}^{K}\gamma_k \;  \textbf{\textit{X}}_k^\top \textbf{\textit{S}}_{k}\;,
\end{equation}
where $z \in \mathbb{R}$, ${\textbf{\textit{X}}}_k \in \mathbb{R}^{5}$ is a properly embedded 3D input point, $\gamma_k \in \mathbb{R}$ is the scale factor, i.e., the last element of the learned parameter vector $\widetilde{\textbf{\textit{S}}}_k$, and $\textbf{\textit{S}}_k = \widetilde{\textbf{\textit{S}}}_k / \gamma_k \in \mathbb{R}^{5}$ are the corresponding normalized learned parameters (spheres). 

Furthermore, \citet{melnyk2020embed} demonstrated that the hypersphere (and geometric) neuron activations are isometric in 3D. That is, rotating the input is equivalent to rotating the decision spheres. This result is a necessary condition to consider rotation and translation equivariance of models constructed with spherical neurons and forms the foundation of our methodology.

In the following sections, we use the same notation for a 3D rotation $\textbf{\textit{R}}$ represented in the Euclidean space $\mathbb{R}^3$, the homogeneous (projective) space $\textit{P}(\mathbb{R}^3)$, and $\mathbb{ME}^{3}\cong\mathbb{R}^5$, depending on the context. This is possible since we can append the required number of ones to the diagonal of the original rotation matrix without changing the transformation representation.

\section{Method}
\label{methodology}
In this section, we identify the conditions under which a spherical neuron as a function of its 3D input can be steered. In other words, we derive an expression that gives us the response of a \textit{hypothetical} spherical neuron for some input, using rotated versions of the learned spherical neuron parameters.
 We start by considering the steerability conditions for a single sphere classifying the corresponding input point $\textbf{\textit{X}}$, i.e., $f(\textbf{\textit{X}})=\textbf{\textit{X}}^\top \textbf{\textit{S}}$, where $\textbf{\textit{X}}$ and $\textbf{\textit{S}}$ are embedded in $\mathbb{ME}^{3}\cong\mathbb{R}^5$ according to (\ref{hypersphere_in_r}). 

\subsection{Basis construction}
\label{steerable_basis}
\begin{figure}
	\centering
	\includegraphics{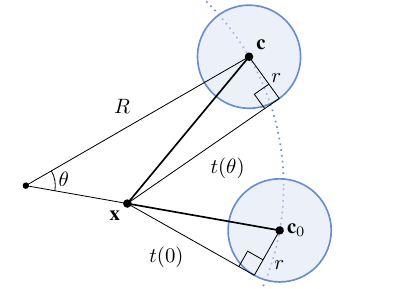}
	\caption{The effect of rotation on the spherical neuron activation in 2D; $t(\theta)$ denotes the tangent length as a function of rotation angle, \ie, $t(0)^2 = \norm{\textbf{x}-\textbf{c}_0}^2 - r^2 = -2\textbf{\textit{X}}^\top \textbf{\textit{S}}$.}
	\label{fig:2d_equivariance}
	\vspace{-5pt}
\end{figure}

To formulate a steerability constraint for a spherical neuron that has one sphere as a decision surface, we first need to determine the minimum number of basis functions, \ie, the number of terms $M$ in (\ref{eq:3d_steerability_constraint}).
This number only depends on the degree of the spherical harmonics that are required to compute the steered result \citep{freeman1991design}. Thus, we need to determine the required degrees.
\begin{theorem}
\label{th:the_theorem}
Let $\textbf{\textit{S}}\in \mathbb{R}^{n+2}$ be an $nD$ spherical classifier with center $\cm_0\in\mathbb{R}^n$ ($R:=\norm{\cm_0}$) and radius $r$, and $\xm\in\mathbb{R}^n$ be a point represented by $\textbf{\textit{X}}\in \mathbb{R}^{n+2}$, see \eqref{hypersphere_in_r}. Let further $\textbf{\textit{S}}'$ be the classifier that is obtained by rotating $\textbf{\textit{S}}$ in $nD$ space (\ie, using an element of $\SO(n)$), then $\textbf{\textit{X}}^\top\textbf{\textit{S}}'$ and $\textbf{\textit{X}}^\top\textbf{\textit{S}}$ are related by spherical harmonics up to first degree.
\end{theorem}
\begin{proof}
Without loss of generality, the rotation is defined by the plane of rotation $\pmb{\pi}$ and the angle $\theta$. Denote the projection of a vector $\textbf{v}\in\mathbb{R}^n$ 
on $\pmb{\pi}$ 
by $\textbf{v}_{\pmb\pi}$ 
and define $\textbf{v}_{\perp{\pmb\pi}}=\textbf{v}-\textbf{v}_{\pmb\pi}$.
From \eqref{scalar_product_isomorphism} we obtain
\[
\begin{aligned}
2\textbf{\textit{X}}^\top \textbf{\textit{S}} &= r^2 - \norm{\textbf{x}-\textbf{c}_{0}}^2  \\&= r^2 - \norm{(\textbf{x}-\textbf{c}_{0})_{\perp\pmb{\pi}}}^2 - \norm{(\textbf{x}-\textbf{c}_{0})_{\pmb{\pi}}}^2\enspace.
\end{aligned}
\]
A rotation in $\pmb{\pi}$ only affects the rightmost term above and there exists a $\phi\in[0,2\pi)$ such that
\[
\begin{aligned}
 \norm{(\textbf{x}-\textbf{c}_{0})_{\pmb{\pi}}}^2 &= \norm{\textbf{x}_{\pmb{\pi}}-\textbf{c}_{0 \, \pmb{\pi}}}^2 \\&= \norm{\textbf{x}_{\pmb{\pi}}}^2+\norm{\textbf{c}_{0 \, \pmb{\pi}}}^2 - 2 \norm{\textbf{x}_{\pmb{\pi}}}\,\norm{\textbf{c}_{0 \, \pmb{\pi}}}\,\cos\phi\enspace.
 \end{aligned}
\]
With a similar argument, we obtain
\[
\begin{aligned}
2\textbf{\textit{X}}^\top \textbf{\textit{S}}' =  &~r^2 - \norm{(\textbf{x}-\textbf{c}_{0})_{\perp\pmb{\pi}}}^2  \\&- \norm{\textbf{x}_{\pmb{\pi}}}^2-\norm{\textbf{c}_{0 \, \pmb{\pi}}}^2 + 2 \norm{\textbf{x}_{\pmb{\pi}}}\,\norm{\textbf{c}_{0 \, \pmb{\pi}}}\,\cos(\phi+\theta).
\end{aligned}
\]
\qedhere
\end{proof}

\begin{figure}
	\centering
	\includegraphics[width=0.9\linewidth]{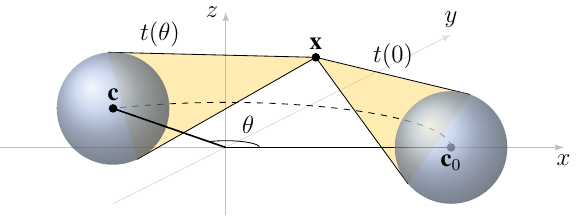}
	\caption{The effect of rotation on the spherical neuron activation in 3D; $t(\theta)$ denotes the tangent length as a function of rotation angle.}
	\label{fig:3d_equivariance}
	\vspace{-5pt}
\end{figure}

This result is valid in any dimension, but we are primarily interested in $n=2$, as illustrated in Figure~\ref{fig:2d_equivariance} for the case $\phi=0$, and $n=3$, shown in Figure~\ref{fig:3d_equivariance}.

Following the result of Theorem 4 in \citet{freeman1991design} and using $N=1$, we have that $M = (N+1)^2 = 4$ basis functions suffice in the 3D case \eqref{eq:3d_steerability_constraint}.

\subsection{Spherical filter banks in 3D}
\label{filter_bank}
In 3D, we thus select four rotated versions of the function $f(\textbf{\textit{X}})$, as the basis functions. The rotations $\{\textbf{\textit{R}}_j\}_{j=1}^{4}$ must be chosen to satisfy the condition (b) in Theorem 4 \citep{freeman1991design}. Therefore, we transform $f(\textbf{\textit{X}})$ such that the resulting four spheres are spaced in three dimensions equally, \ie, form a regular tetrahedron with the vertices $(1,1,1)$, $(1,-1,-1)$, $(-1, 1,-1)$, and $(-1, -1, 1)$, as shown in Figure~\ref{fig:tetrahedron}. We stack the homogeneous coordinates of the tetrahedron vertices $\textbf{m}_j$ in a matrix column-wise (scaled by $1/2$) to get the orthogonal matrix
\begin{equation}
	\label{eq:basis_matrix}
	\textbf{M}  = 
	\begin{bmatrix}
		\textbf{m}_1 & \textbf{m}_2 & \textbf{m}_3 & \textbf{m}_4	\end{bmatrix} = \frac{1}{2}
	\begin{bmatrix}
		1 &  \phantom{-}1 & -1             &  -1   \\
		1 & -1            &  \phantom{-}1  &  -1   \\
	    1 & -1            & -1             &  \phantom{-}1   \\
	    1 &  \phantom{-}1 &  \phantom{-}1  &  \phantom{-}1   \\
	\end{bmatrix} ~.
\end{equation}
We will use this matrix operator $\textbf{M}$ to compute the linear coefficients in the vector space generated by the vertices of the regular tetrahedron \citep{diva2:302469}. This will be necessary to find the appropriate interpolation functions and formulate the steerability constraint in Section~\ref{our_steerability_constraint}.

The four rotated versions of the function  $f(\textbf{\textit{X}})$ will constitute the basis functions that we call a spherical filter bank. To construct this filter bank, we choose the following convention. The originally learned spherical classifier  $f(\textbf{\textit{X}})=\textbf{\textit{X}}^\top \textbf{\textit{S}}$ is first rotated to $\norm{\textbf{c}_0}\cdot(1,1,1)$ (see Figure~\ref{fig:tetrahedron}) with the corresponding (geodesic) transformation denoted as $\textbf{\textit{R}}_O$. Next, we rotate the transformed sphere into the other three vertices of the regular tetrahedron and transform back to the original coordinate system (see Figure~\ref{fig:diagram} for the case of $(1, -1, -1)$).

\begin{figure}
	\centering\hfill
	\includegraphics[height=30mm]{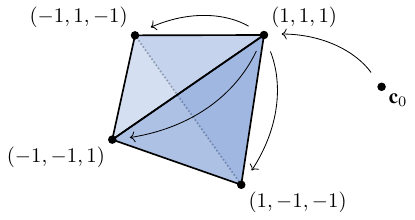}\hfill\mbox{}
	\caption{A regular tetrahedron as a basis. Without loss of generality, assume $\norm{\textbf{c}_0} = \sqrt{3}$.}
	\label{fig:tetrahedron}
	\vspace{-5pt}
\end{figure}

The resulting filter bank for one spherical classifier is thus composed as the following $4 \times 5$ matrix:
\begin{equation}
	\label{eq:sphere_filter_bank}
	B(\textbf{\textit{S}}) = 
	\begin{bmatrix}
		\textbf{\textit{R}}_O^{\top}\, \textbf{\textit{R}}_{T_i}\, \textbf{\textit{R}}_O\, \textbf{\textit{S}} \\
	\end{bmatrix}_{i={0\ldots3}} ~,
\end{equation}
where each of $\{\textbf{\textit{R}}_{T_i}\}_{i=0}^{3}$ is the rotation isomorphism in $\mathbb{R}^5$ corresponding to a 3D rotation from $(1,1,1)$ to the vertex $i+1$ of the regular tetrahedron. Therefore, $\textbf{\textit{R}}_{T_0}=\textbf{I}_5$.

As we show by Theorem~\ref{th:the_second_theorem}, the spherical filter bank $B(\textbf{\textit{S}})$ is equivariant under 3D rotations.

\begin{theorem}
\label{th:the_second_theorem}
Let $\textbf{\textit{S}}\in \mathbb{R}^{5}$ be a 3$D$ spherical classifier (\ie, the hypersphere neuron) and $\xm\in\mathbb{R}^3$ be an input point represented by $\textbf{\textit{X}}\in \mathbb{R}^{5}$, see \eqref{hypersphere_in_r}. Let further 
$B(\textbf{\textit{S}})\in \mathbb{R}^{4 \times 5}$ be a filter bank obtained according to \eqref{eq:sphere_filter_bank}. 
Then the filter bank output 
$\ym = B(\textbf{\textit{S}}) \textbf{\textit{X}} \in \mathbb{R}^{4}$ is equivariant to 3D rotations of $\xm$.
\end{theorem}
\begin{proof}
Since rotations in $\mathbb{R}^3$ are embedded into $\mathbb{R}^5$ acting on the first three components (for  details, see \citet{melnyk2020embed}), we need to show that the left sub-matrix $B_3(\textbf{\textit{S}})\in \mathbb{R}^{4 \times 3}$ of $B(\textbf{\textit{S}})$ is of rank 3. By construction, the four rows of $B_3(\textbf{\textit{S}})$ form the vertices of a tetrahedron, thus span $\mathbb{R}^3$.
\qedhere
\end{proof}
Note that the right sub-matrix $B_2(\textbf{\textit{S}})\in \mathbb{R}^{4 \times 2}$ of $B(\textbf{\textit{S}})$ results in a constant vector $B_2(\textbf{\textit{S}})[x_4, x_5]^\top \in \mathbb{R}^{4}$ independent of the rotation, where $x_4$ and $x_5$ are the last two components of $\textbf{\textit{X}}$ (see \eqref{hypersphere_in_r}). This additional constant vector requires the use of the fourth vertex of the tetrahedron.

Furthermore, we can explicitly show how the representation $V_{\textbf{\textit{R}}} \in \mathbb{R}^{4\times4}$ of the 3D rotation \textbf{\textit{R}} can be obtained in the filter bank $B(\textbf{\textit{S}})$ output space.

\begin{proposition}
Let $\textup{\textbf{M}}$ be the $4\times4$ orthogonal matrix defined in \eqref{eq:basis_matrix}, $\textbf{\textit{R}}_O^k \in \mathbb{R}^{4\times4}$ be the $\textit{P}(\mathbb{R}^3)$ representation (constructed by appending ones to the main diagonal -- see the last paragraph in Section~3.3) of the 3D rotation (geodesic) from the center $\cm_0^k$ of the learned \textup{k}th spherical classifier $\textbf{\textit{S}}_k$ to $\norm{\cm_0^k} \cdot (1,1,1)$, and, slightly abusing notation, $\textbf{\textit{R}} \in \mathbb{R}^{4\times4}$ be the $\textit{P}(\mathbb{R}^3)$ representation of the 3D rotation $\textbf{\textit{R}}$.

Then $V^k_{\textbf{\textit{R}}} \in \mathbb{R}^{4\times4}$ defined below is the representation of the 3D rotation \textbf{\textit{R}} in the filter bank $B(\textbf{\textit{S}}_k)$ output space:
\begin{equation}
	\label{eq:V_for_one_point}
	V^k_{\textbf{\textit{R}}} = \textup{\textbf{M}}^\top \textbf{\textit{R}}_O^k\, \textbf{\textit{R}}\, \textbf{\textit{R}}_O^{k\top} \textup{\textbf{M}} ~.
\end{equation}
\end{proposition}
\begin{proof}
Since $\det{\textup{\textbf{M}}} = 1$, we have that $\textup{\textbf{M}} \in \SO(4)$, and thus, all terms on the right hand-side of \eqref{eq:V_for_one_point} are $4\times4$ rotation matrices with only one variable --- $\Rb$. 
Therefore, $V^k_{\Rb}$ is also a rotation matrix and is a unique map. The inverse transformation can be computed straightforwardly as
\begin{equation}
	\label{eq:R_to_V}
	\textbf{\textit{R}} = \textbf{\textit{R}}_O^{k^\top} \textbf{M}\, V^k_{\textbf{\textit{R}}}\, \textbf{M}^{\top} \textbf{\textit{R}}_O^{k}
\end{equation}
with subsequent extraction of the upper-left $3\times3$ sub-matrix as the original 3D rotation matrix $\textbf{\textit{R}}$.
Hence, there is a one-to-one mapping between $V^k_{\Rb}$ and $\Rb$.
\end{proof}
Thus, both $V_{\Rb}$ and $\Rb$ are representations of the 3D rotation, and we can write the rotation equivariance property of the spherical filter bank $B(\textbf{\textit{S}})$ as
\begin{equation}
\label{eq:filter_bank_equivariance}
    V_{\Rb} \, B(\textbf{\textit{S}}) \, \textbf{\textit{X}} = B(\textbf{\textit{S}})\,\textbf{\textit{R}}\textit{\textbf{X}}.
\end{equation}

\begin{figure}
	\centering
	\includegraphics[height=30mm]{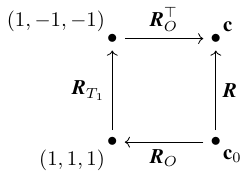}\hfill%
	\caption{A rotation from $\textbf{c}_0$ to $\textbf{c}$ described by a tetrahedron rotation. Without loss of generality, assume $\norm{\textbf{c}_0} = \sqrt{3}$.}
	\label{fig:diagram}
	\vspace{-5pt}
\end{figure}

\subsection{3D steerability constraint}
\label{our_steerability_constraint} 
The steerability constraint can be formulated as follows. For an arbitrary rotation $\textbf{\textit{R}}$ applied to the input of the function $f(\textbf{\textit{X}})$, we want the output of the spherical filter bank $B(\textbf{\textit{S}})$ in \eqref{eq:sphere_filter_bank} to be interpolated with $v_j(\textbf{\textit{R}})$ such that the response is equal to the original function output, \ie, 
\begin{equation}
	\label{eq:our_steerability_constraint}
	\begin{aligned}
	f(\textbf{\textit{X}}) = 	f^{\textbf{\textit{R}}}(\textbf{\textit{R}}\textbf{\textit{X}}) &= \sum_{j=1}^{M}v_j(\textbf{\textit{R}})\, f^{\textbf{\textit{R}}_{j}}(\textbf{\textit{R}}\textbf{\textit{X}})~\\&= v(\textbf{\textit{R}})^\top B(\textbf{\textit{S}})\,\textbf{\textit{R}}\textit{\textbf{X}},
	\end{aligned}
\end{equation}
where $\textit{\textbf{X}} \in \mathbb{R}^5$ is a single, appropriately embedded, 3D point, and $v(\textbf{\textit{R}}) \in \mathbb{R}^4$ is a vector of the interpolation coefficients.

The coefficients $v(\textbf{\textit{R}})$ should conform to the basis function construction (condition (b) in Theorem~4 in \citet{freeman1991design}), which is why they are computed with \textbf{M} defined in (\ref{eq:basis_matrix}).
Given $\textbf{\textit{X}} \in \mathbb{R}^5$ as input and an unknown rotation $\textbf{\textit{R}}$ acting on it, the steering equation (\ref{eq:our_steerability_constraint}) implies
\begin{equation}
	\label{eq:steerability_constraint_for_one_point}
v(\textbf{\textit{R}})^\top B(\textbf{\textit{S}})\,\textbf{\textit{R}} \textit{\textbf{X}} =  v(\textbf{I})^\top B(\textbf{\textit{S}})\,\textit{\textbf{X}}~.
\end{equation}
Given a tetrahedron rotation, e.g., $\textbf{\textit{R}}_{T_1}$, as shown in the diagram in Figure~\ref{fig:diagram}, we can define the unknown rotation accordingly as $\textbf{\textit{R}} = \textbf{\textit{R}}_O^{\top} \textbf{\textit{R}}_{T_1} \textbf{\textit{R}}_O$.
In this case, it is easy to see that to satisfy the constraint (\ref{eq:steerability_constraint_for_one_point}),   $v(\textbf{\textit{R}})$ must be $(0, 1, 0, 0)$, i.e., the second filter in the filter bank $B(\textbf{\textit{S}})$ must be chosen. This can be achieved by transforming a constant vector $\textbf{m}_1 = (\frac{1}{2}, \frac{1}{2}, \frac{1}{2}, \frac{1}{2})$ by rotation $\textbf{\textit{R}}_{T_1}$ and multiplying it by the basis matrix $\textbf{M}$ as follows:
\begin{equation}
	\label{eq:vs_for_rt1}
	\begin{aligned}
	v(\textbf{\textit{R}}) &= \textbf{M} ^\top (\textbf{\textit{R}}_{T_1}\,\textbf{m}_1) \\&= \textbf{M}^\top \Big(\frac{1}{2}, -\frac{1}{2}, -\frac{1}{2}, \frac{1}{2}\Big)  = (0, 1, 0, 0)~.
	\end{aligned}
\end{equation}
Note that, in general, a geometric neuron (\ref{eq:geometric_neuron}) takes a \textit{set} of embedded points as input. 
Therefore, with the setup above, the rotation of $\textbf{m}_1$ will be different for each input shape point $k$ if the same $v$ is used for all $k$, which
contradicts that the shape is transformed by a rigid body motion, i.e., the same $\textbf{\textit{R}}$ for all $k$.
Thus, we need to consider a suitable vector $v^k$ for each input shape point $k$, such that the resulting $\textbf{\textit{R}}$ is the same for all $k$. 
This can be achieved by recalling how we construct the basis functions in the spherical filter bank (\ref{eq:sphere_filter_bank}): we need to consider the respective initial rotation $\textbf{\textit{R}}_O^k$. The desired interpolation coefficients $v^k$ are thus computed as
\begin{equation}
	\label{eq:vs_for_one_point}
	v^k(\textbf{\textit{R}}) = \textbf{M}^\top \,(\textbf{\textit{R}}_O^k\, \textbf{\textit{R}}\, \textbf{\textit{R}}_O^{k\top} \textbf{m}_1) ~.
\end{equation}
The resulting $v^k(\textbf{\textit{R}}) \in \mathbb{R}^4$ interpolate the responses of the tetrahedron-copies of the originally learned sphere $\textbf{\textit{S}}_k$ to replace the rotated sphere.

Remarkably, the interpolation vector $v^k(\Rb)$ is the first column of $V^k_{\textbf{\textit{R}}}$ \eqref{eq:V_for_one_point}. 
Note that $v^k(\textbf{\textit{R}})$ has only three degrees of freedom since the four vertices of the regular tetrahedron sum to zero (see Figure~\ref{fig:tetrahedron}). Thus, $v^k(\textbf{\textit{R}})$ is equivariant under 3D rotations.
 
By plugging \eqref{eq:our_steerability_constraint} into \eqref{eq:geometric_neuron}, we can now establish the steerability constraint for a geometric neuron, $f_{\text{GN}}$, which takes a set of $K$ embedded points $\{\textbf{\textit{X}}_1,\ldots,\textbf{\textit{X}}_k\}$ as input: 
\begin{equation}
	\label{eq:gn_steerability_constraint}
	\begin{aligned}
	f_{\text{GN}}^{\textbf{\textit{R}}}(\textbf{\textit{R}} \textbf{\textit{X}}) &= \sum_{k=1}^{K}\gamma_k\, f^{k\,\textbf{\textit{R}}}(\textbf{\textit{R}} \textbf{\textit{X}}_k) \\&= \sum_{k=1}^{K}\gamma_k\, v^k(\textbf{\textit{R}})^\top B(\textbf{\textit{S}}_{k})\,\textbf{\textit{R}}\textit{\textbf{X}}_k~.
	\end{aligned}
\end{equation}

 \subsection{Steerable model overview}
 To build a 3D steerable model, we perform the following steps (see Figure~\ref{fig:model_overview}): We first train an ancestor model, which consists of spherical neurons. After training the model parameters to classify data in a canonical orientation, we freeze them and transform the first layer weights according to the 3D steerability constraint \eqref{eq:our_steerability_constraint} (for hypersphere neurons) or \eqref{eq:gn_steerability_constraint} (for geometric neurons). Finally, by combining the resulting parameters in spherical filter banks according to \eqref{eq:sphere_filter_bank} and adding interpolation coefficients \eqref{eq:vs_for_one_point} as free parameters, we create a steerable model. If the interpolation coefficients are computed correctly, the model will produce rotation-invariant predictions.

\section{Experiments}
\label{demonstration}
In this section, we describe the experiments we conducted to confirm our findings presented in Section~\ref{methodology}.
\vspace{-2pt}

\subsection{Datasets}
\label{datasets}
\paragraph{3D Tetris} Following the experiments reported by \citet{thomas2018tensor}, \citet{weiler20183d}, and \citet{melnyk2020embed}, we use the following synthetic point set of eight 3D Tetris shapes \citep{thomas2018tensor} consisting of four points each (see, \eg, Figure 3 in \citet{melnyk2020embed}):

{\tiny
\begin{verbatim}
chiral_shape_1: [(0, 0, 0), (0, 0, 1), (1, 0, 0), (1,  1, 0)],
chiral_shape_2: [(0, 0, 0), (0, 0, 1), (1, 0, 0), (1, -1, 0)],
square:         [(0, 0, 0), (1, 0, 0), (0, 1, 0), (1,  1, 0)],
line:           [(0, 0, 0), (0, 0, 1), (0, 0, 2), (0,  0, 3)],
corner:         [(0, 0, 0), (0, 0, 1), (0, 1, 0), (1,  0, 0)],
L:              [(0, 0, 0), (0, 0, 1), (0, 0, 2), (0,  1, 0)],
T:              [(0, 0, 0), (0, 0, 1), (0, 0, 2), (0,  1, 1)],
zigzag:         [(0, 0, 0), (1, 0, 0), (1, 1, 0), (2,  1, 0)].
\end{verbatim}
}

\paragraph{3D skeleton data} We also perform experiments on real-world data to substantiate the validity of our theoretical results. We use the UTKinect-Action3D dataset introduced by \citet{xia2012view}, in particular, the 3D skeletal joint locations extracted from Kinect depth maps. For each action sequence and from each frame, we extract the skeleton consisting of twenty points and assign it to the class of actions (ten categories) this frame is part of. Therefore, we formulate the task as shape recognition (\ie, a kind of action recognition from a static point cloud), where each shape is of size $20 \times 3$.

Since the orientations of the shapes vary significantly across the sequences, we perform the following standardization: We first center each shape at the origin, and then compute the orientation from its three hip joints and de-rotate the shape in the \textit{xy}-plane (viewer coordinate system) accordingly. We illustrate the effect of de-rotation in Figure~\ref{fig:skeleton_data}.
From each action sequence, we randomly select 50\% of the skeletons for the test set and 20\% of the remainder as validation data. The resulting data split is as follows: 2295 training shapes, 670 shapes for validation, and 3062 test shapes, approximately corresponding to $38\%, 11\%, 51\%$ of the total amount of the skeletons, respectively.
\vspace{-2pt}
 \begin{figure}
	\centering
    \includegraphics[width=0.49\linewidth,trim={7cm 7cm 7cm 7cm},clip]{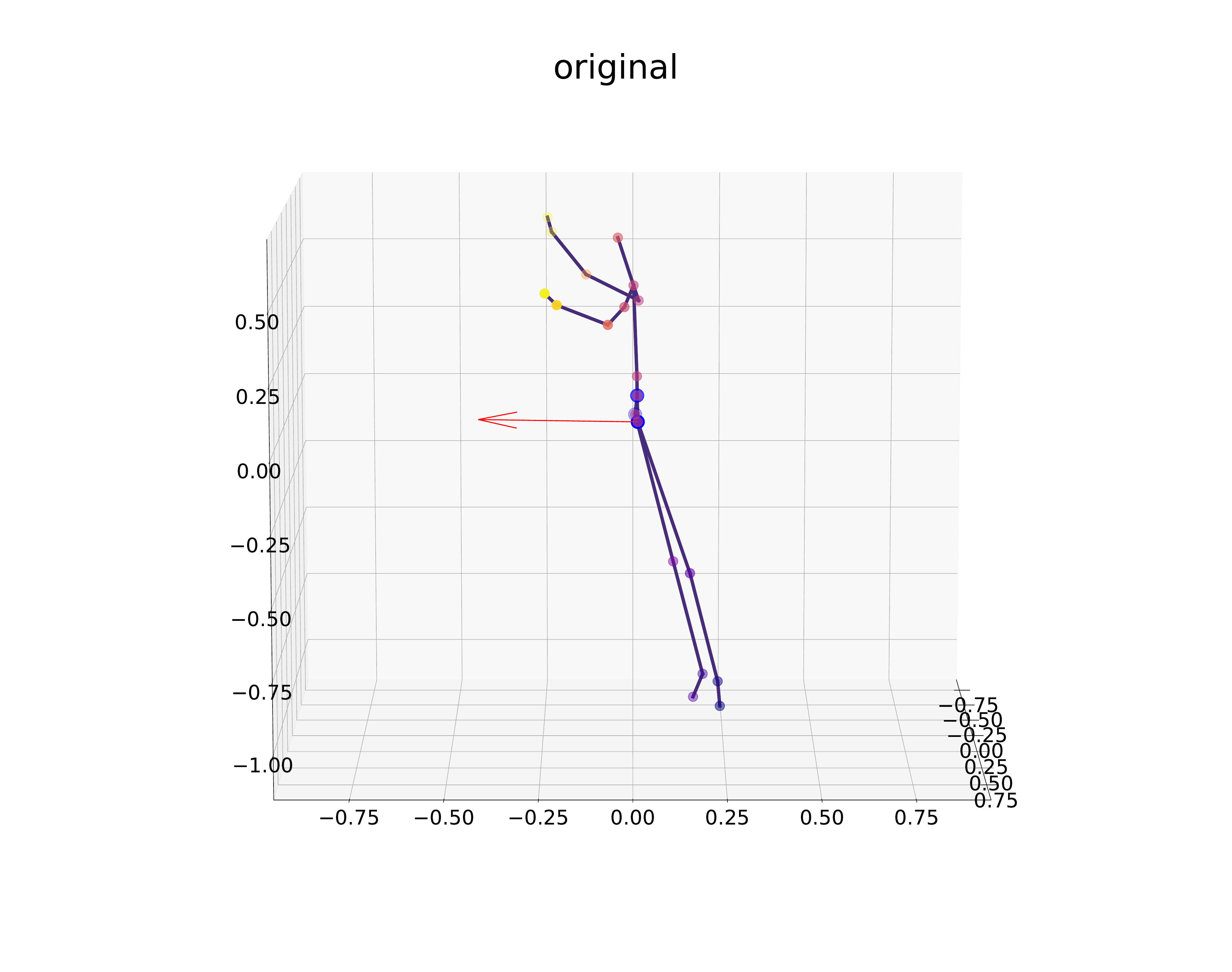}
	\includegraphics[width=0.49\linewidth,trim={7cm 7cm 7cm 7cm},clip]{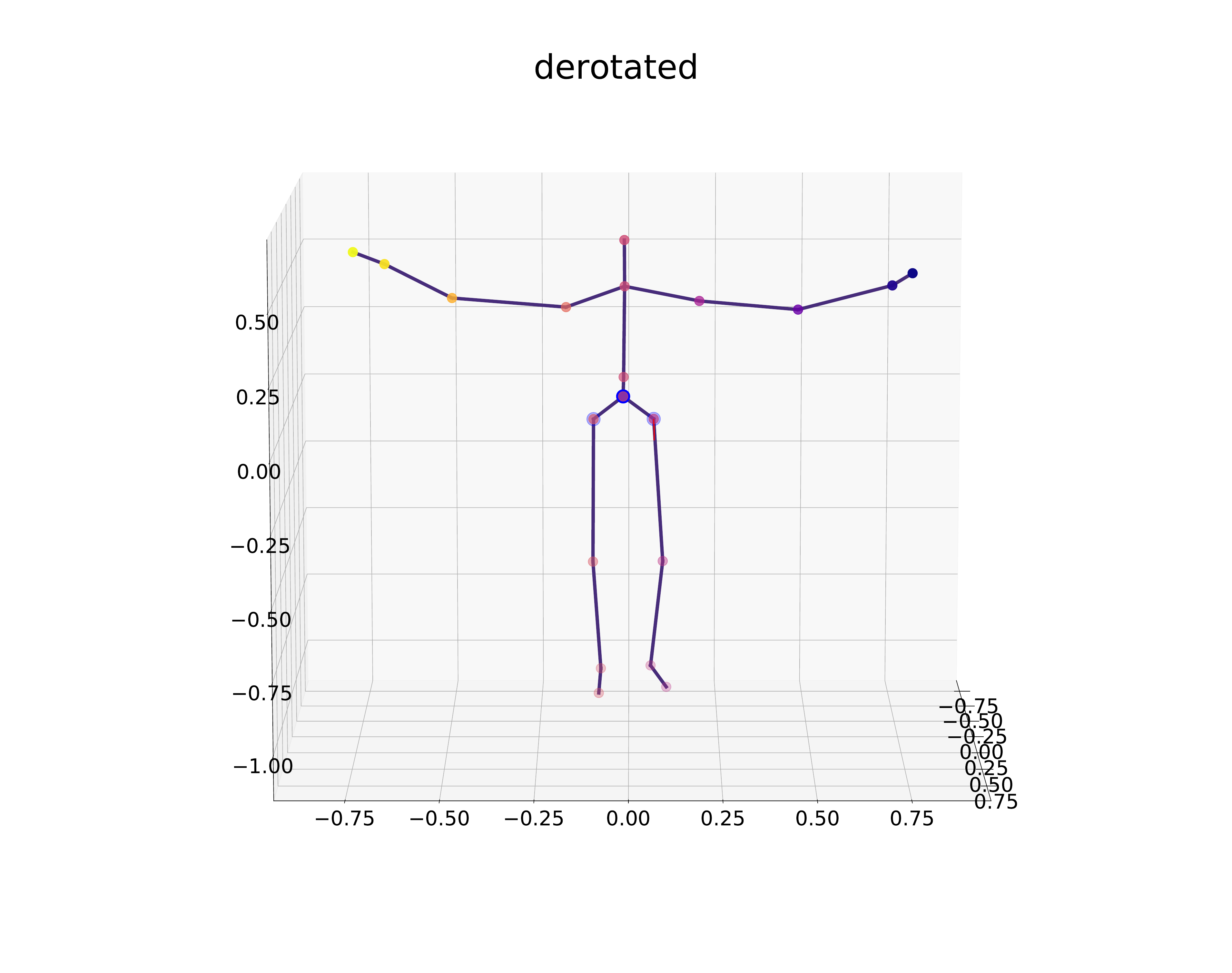}
	\caption{The effect of standardizing the orientation of the 3D skeleton representing \textit{waveHands} action: the original (left) and the de-rotated (right) shape; the arrow is the normal vector of the plane formed by the three hip joints.}
	\label{fig:skeleton_data}
	\vspace{-9pt}
\end{figure}

\subsection{Steerable model construction}
\label{steerable_model_construction}
To construct and test steerable models, we perform the same steps for both datasets. The minor differences are the choice of training hyperparameters and the presence of validation and test subsets in the 3D skeleton dataset.
\begin{table*}[t]
	\small
	\centering
	\caption{The steerable model classification accuracy for the distorted (the noise units are specified in the square brackets) rotated shapes and the ancestor accuracy for the distorted shapes in their canonical orientation (mean and std over 1000 runs, \%).}
	\label{tab:known_r_experiment}
	\begin{tabular}{l@{}cccl@{}cc}\\
		\toprule
		\multicolumn{3}{c}{3D Tetris} &&  \multicolumn{3}{c}{3D skeleton data (\textit{test} set)}\\
		\cmidrule(l){1-3} \cmidrule(l){5-7}
		Noise ($a$), [1]~~~~& \textit{Steerable} & \textit{Ancestor} && Noise ($a$), [m]~~~~& \textit{Steerable} & \textit{Ancestor}\\
		\midrule 
		 $0.00$ & $100.0\pm0.0$  & $100.0\pm0.0$ &&  $0.000$  &  $92.9\pm0.0$  & $92.9\pm0.0$ \\
		 
		 $0.05$ & $100.0\pm0.0$  & $100.0\pm0.0$ &&  $0.005$ &  $92.4\pm0.2$  & $92.4\pm0.2$\\
		 
		 $0.10$ & $100.0\pm0.0$  & $100.0\pm0.0$ &&  $0.010$ &  $91.1\pm0.3$  & $91.1\pm0.3$\\
		 
		 $0.20$ & $100.0\pm0.4$  & $100.0\pm0.0$ &&  $0.020$ &  $87.1\pm0.5$  & $87.1\pm0.5$\\
		 
		 $0.30$ & $99.7\pm1.9$   & $99.8\pm1.6$ &&  $0.030$ &  $82.3\pm0.6$  & $82.2\pm0.6$\\
		 
		 $0.50$ & $94.9\pm7.7$   & $95.0\pm7.9$ &&  $0.050$ &  $72.0\pm0.7$   & $71.9\pm0.7$\\
		\bottomrule
	\end{tabular}	
\end{table*}
\begin{table*}[htb!]
	\small
	\centering
	\caption{The L1 distance between the steerable model hidden activations and the ground truth activations given the distorted (the noise units are specified in the square brackets) rotated shapes and between the ancestor hidden activations given the distorted shapes in their canonical orientation (mean and std over 1000 runs).}
	\label{tab:known_r_experiment_l1}
	\begin{tabular}{l@{}cccl@{}cc}\\
		\toprule
		\multicolumn{3}{c}{3D Tetris} &&  \multicolumn{3}{c}{3D skeleton data (\textit{test} set)}\\
		\cmidrule(l){1-3} \cmidrule(l){5-7}
		Noise ($a$), [1]~~~~& \textit{Steerable} & \textit{Ancestor} && Noise ($a$), [m]~~~~&\textit{Steerable}& \textit{Ancestor}\\	
		\midrule 
		 $0.00$ & $0.00\pm0.00$  & $0.00\pm0.00$ &&    $0.000$  &  $0.00\pm0.00$  & $0.00\pm0.00$ \\
		 $0.05$ & $0.33\pm0.05$  & $0.33\pm0.05$ &&    $0.005$ &  $0.53\pm0.00$  & $0.53\pm0.00$\\
		 $0.10$ & $0.66\pm0.10$  & $0.66\pm0.10$ &&    $0.010$ &  $1.06\pm0.01$  & $1.06\pm0.01$\\
		 $0.20$ & $1.32\pm0.19$  & $1.32\pm0.19$ &&    $0.020$ &  $2.12\pm0.01$  & $2.12\pm0.01$\\
		 $0.30$ & $2.00\pm0.31$  & $2.00\pm0.29$ &&    $0.030$ &  $3.18\pm0.01$  & $3.18\pm0.01$\\
		 $0.50$ & $3.33\pm0.48$  & $3.32\pm0.48$ &&    $0.050$ &  $5.30\pm0.02$  & $5.30\pm0.02$\\
		\bottomrule
	\end{tabular}	
\end{table*}

We first train a two-layer (\textit{ancestor}) multilayer geometric perceptron (MLGP) model \citep{melnyk2020embed}, where the first layer consists of geometric neurons and the output layer of hypersphere neurons, to classify the shapes. 
Since the architecture choice is not the objective of the experiments, when building the MLGP, we use only one configuration with five hidden units for the Tetris data and twelve hidden units (determined by using the validation data) for the 3D skeleton dataset throughout the experiments. 

Similar to \citet{melnyk2020embed}, we do not use any activation function in the first layer due to the nonlinearity of the conformal embedding.
We implement both MLGP models in PyTorch \citep{paszke2019pytorch} and keep the default parameter initialization for the linear layers. 
We train both models by minimizing the cross-entropy loss function and use the Adam optimizer \citep{kingma2014adam} with the default hyperparameters (the learning rate is $0.001$). 
The Tetris MLGP learns to classify the eight shapes in the canonical orientation perfectly after 2000 epochs, whereas the Skeleton model trained for 10000 epochs achieves a test set accuracy of 92.9\%. 

For both, we then freeze the trained parameters and construct a steerable model. Note that we form steerable units only in the first layer and keep the output, \ie, classification, layer hypersphere neurons as they are. The steerability is not required for the subsequent layers as the output of the first layer becomes rotation-independent. The steerable units are formed from the corresponding frozen parameters as the (fixed) filter banks according to (\ref{eq:sphere_filter_bank}).

The only free parameters of this constructed steerable model are interpolation coefficients $v^k(\textbf{\textit{R}}) \in \mathbb{R}^4$ defined in (\ref{eq:vs_for_one_point}), where $k$ indexes the learned first layer parameters (spheres) in the ancestor MLGP model. 

Note that the interpolation coefficients (\ref{eq:vs_for_one_point}) are all parameterized by the input orientation $\textbf{\textit{R}}$. If the interpolation coefficients are computed correctly, the model output will be rotation-invariant.
However, in this paper, we focus on the steerability as such, and a practical way of determining $v^k(\textbf{\textit{R}})$ remains to be done in future work, including the experimental comparison with rotation-invariant classifiers from related work. Therefore, the following experiment serves as an empirical validation of the derived constraint rather than rotation-invariant predictions.
\vspace{-5pt}

\subsection{Known rotation experiment}
\label{known_r_experiment} 
Using the trained ancestor MLGP, we verify the correctness of \eqref{eq:gn_steerability_constraint} (and, therefore, \eqref{eq:our_steerability_constraint}). We first rotate the original data and then use this ground truth rotation to compute the interpolation coefficients of the constructed steerable model according to (\ref{eq:vs_for_one_point}). Our intuition is that if the steerability constraint (\ref{eq:gn_steerability_constraint}) is correct, then, given the transformed point set, the activations of the steerable units in the steerable model will be equal to the activations of the geometric neurons in the ancestor MLGP model fed with the point set in the canonical orientation. Hence, the classification accuracies of the ancestor and steerable models on the original and transformed datasets, respectively, should be equal.

We run this experiment 1000 times. Each time, we generate a random rotation and apply it to the original point set (in case of the 3D skeleton data, we use the \textit{test} split). We use this ground-truth rotation information to compute the interpolation coefficients for the steerable model, which we then evaluate on the transformed point set. 

To verify the stability of the steerable unit activations, we add uniform noise to the transformed points, $\textbf{\textit{n}} \sim U(-a, a)$, where the range of \textit{a} is motivated by the magnitude of the points in the datasets: for the Tetris data, the highest $a$ is chosen to be $0.5$, making it impossible for the noise to perturb a shape of one category into a shape of another class; for the skeleton data, the highest $a$ is set to $0.05$~m -- a reasonable amount of distortion for a human of average size (see, e.g., Figure~\ref{fig:skeleton_data}), which should be insufficient to completely change the representation of the skeleton action.  

For reference, we also present the accuracy of the ancestor model classifying the data in the canonical orientation. We summarize the results in Table~\ref{tab:known_r_experiment}. 
Additionally, by computing the L1 distance, we compare the hidden unit activations of the ancestor MLGP fed with the shapes in the canonical orientation (\textit{ground truth} activations) and those of the constructed steerable models fed with the rotated data (see Table~\ref{tab:known_r_experiment_l1}).

\section{Discussion and Conclusion}
\label{discussion}
Enabled by the complete understanding of the geometry of the spherical neurons, we show in Section~\ref{methodology} that we only need the spherical harmonics of degree up to $N=1$ to determine the effect of rotation on the activations in 3D. Using this result, we derive a novel 3D steerability constraintt~\eqref{eq:our_steerability_constraint} (and \eqref{eq:gn_steerability_constraint}), adding to the practical value of prior work \citep{banarer2003hypersphere, perwass2003spherical, melnyk2020embed}. 

The experiment conducted in Section~\ref{known_r_experiment} shows that the derived constraint is correct since the constructed steerable model produces equally (up to a numerical precision) accurate predictions for the rotated shapes as the ancestor for the shapes in the canonical orientation, provided that the interpolation coefficients are computed with the \textit{known} rotation.
From Table~\ref{tab:known_r_experiment}~and~~\ref{tab:known_r_experiment_l1}, we can see that the steerable model classification error and the L1 distance to the ground truth activations only moderately increase with the level of noise in the input data, which is a clear indication of the robustness of the classifier.

Noteworthy, invariance to point permutations in the input is not attained with the ancestor MLGP model that we used in the experiments. However, one can address the problem of permutation invariance quite straightforwardly: For example, one could follow the construction of, \eg, PointNet \citep{qi2017pointnet}, and replace the shared MLPs applied to each point in isolation with shared spherical neurons taking one point as input (\ie, hypersphere neurons), and then apply a global aggregation function to the extracted features. One would then apply the steerability constraint \eqref{eq:our_steerability_constraint} to the first layer learned parameters to create a steerable model.

In general, one can build an ancestor model containing spherical neurons only in the first layer, with subsequent layers of various flavors (see Figure~\ref{fig:model_overview}). 
As we mentioned in Section~\ref{steerable_model_construction}, steerability is not required for the subsequent layers as the output of the first layer ($\textbf{h}$ in Figure~\ref{fig:model_overview}) becomes rotation-independent, given that the interpolation coefficients are computed correctly.

The interpolation coefficients $v^k$ could be learned directly, \eg, by using a regression network. 
However, this would be a much more complex learning problem than determining $v^k$ \textit{indirectly} by regressing $\Rb$, since all $K$ interpolation vectors $v^k(\Rb)$ are by definition uniquely parameterized by $\Rb$. 

Rotation regression with NNs is per se an important and continuously studied problem that is often considered in the context of pose estimation \citep{zhou2019continuity, Chen_2022_CVPR}. 
Therefore, it is more efficient and intuitive to combine a rotation regressor with our steerable approach for creating a rotation-invariant classifier as compared to the approaches where the rotation information is naively encompassed in the learned parameters of a generic network by training on rotation-augmented data. 

Finally, the main theoretical results of this paper provide a rigorous and geometrically lucid mechanism for constructing steerable feature extractors, the practical utility of which can be fully unveiled in future work.

\section*{Acknowledgments}
	This work was supported by the Wallenberg AI, Autonomous Systems and Software Program (WASP), by the Swedish Research Council through a grant for the project Algebraically Constrained Convolutional Networks for Sparse Image Data (2018-04673), and the strategic research environment ELLIIT.

\bibliography{references}
\bibliographystyle{fullname_bibstyle}

\end{document}